%% file: arxiv_main.tex
\documentclass[11pt]{article}

\usepackage{ifthen}
\newboolean{twocolumnmode}
\setboolean{twocolumnmode}{false}

\title{Causality-Inspired Robustness for Nonlinear Models via Representation
Learning}
\author{Marin Sola, Peter B\"uhlmann, Xinwei Shen}
\date{}

\usepackage{microtype}
\usepackage{subfigure}

\usepackage{natbib}
\usepackage{comment}
\usepackage{lipsum}

\usepackage{mathtools}
\usepackage{tikz}

\newcommand{\Ex}{\mathbb{E}}
\newcommand{\I}{\textbf{I}}
\newcommand{\Bb}{\textbf{B}}
\newcommand{\Pp}{\mathbb{P}}
\newcommand{\supv}{\sup_{v\in C^{\gamma}}\mathbb{E}}
\newcommand{\Xx}{\mathcal{X}}
\newcommand{\Le}{\mathbb{L}^2(\mathcal{X})}
\newcommand{\R}{\mathbb{R}}
\newcommand{\E}{\mathcal{E}}

\DeclareMathOperator*{\cov}{cov}

\DeclareMathOperator*{\argmin}{arg\,min}

\usepackage{hyperref}

\usepackage{amsmath, amssymb, amsthm, mathrsfs, bm}
\usepackage{geometry, extarrows, xcolor, graphicx, enumitem, booktabs, subfigure}
\usepackage{natbib}
\usepackage{titlesec}
\usepackage{hyperref}
\usepackage{array}
\newcolumntype{A}{>{\centering\arraybackslash}m{0.12\columnwidth}}
\newcolumntype{B}{>{\centering\arraybackslash}m{0.4\columnwidth}}

\usepackage[capitalize,noabbrev]{cleveref}

\usepackage{fullpage}

\theoremstyle{plain}
\newtheorem{theorem}{Theorem}
\newtheorem{proposition}[theorem]{Proposition}
\newtheorem{lemma}[theorem]{Lemma}
\newtheorem{corollary}[theorem]{Corollary}
\newtheorem{assumption}{Assumption}
\newtheorem{definition}{Definition}

\theoremstyle{remark}

\newtheorem{remark}{Remark}

\begin{document}
\maketitle

\input{body}

\bibliography{example_paper.bib}

\input{Appx}

\end{document}

%% file: body.tex
\newcommand{\conditionaland}{\ifthenelse{\boolean{twocolumnmode}}{&}{}}
\newcommand{\conditionalnextrow}{\ifthenelse{\boolean{twocolumnmode}}{\\}{}}
\newcommand{\conditionalwidth}{\ifthenelse{\booolean{twocolumnmode}}{\columnwidth}{\textwidth}}
\newcommand{\conditionalheiht}{\ifthenelse{\booolean{twocolumnmode}}{}{}}

\begin{abstract}

Distributional robustness is a central goal of prediction algorithms due to the prevalent distribution shifts in real-world data. The prediction model aims to minimize the worst-case risk among a class of distributions, a.k.a., an uncertainty set. Causality provides a modeling framework with a rigorous robustness guarantee in the above sense, where the uncertainty set is data-driven rather than pre-specified as in traditional distributional robustness optimization. However, current causality-inspired robustness methods possess finite-radius robustness guarantees only in the linear settings, where the causal relationships among the covariates and the response are linear. In this work, we propose a nonlinear method under a causal framework by incorporating recent developments in identifiable representation learning and establish a distributional robustness guarantee. To our best knowledge, this is the first causality-inspired robustness method with such a finite-radius robustness guarantee in nonlinear settings. Empirical validation of the theoretical findings is conducted on both synthetic data and real-world single-cell data, also illustrating that finite-radius robustness is crucial. 

\end{abstract}

\section{Introduction}
In real-world applications, data distributions often shift between training and deployment environments, leading to degraded performance of machine learning (ML) models. These shifts can arise from changes in data collection methods, environmental conditions, or adversarial perturbations. Distributional robustness addresses this challenge by ensuring that models perform well across a range of possible distributions, rather than just the training distribution. This is particularly critical in high-stakes domains such as healthcare, finance, and autonomous systems, where unreliable predictions can have severe consequences. By focusing on robustness, we aim to build models that generalize reliably under some distribution shifts.
The goal of distributional robustness, specifically, is to optimize models for the worst-case scenario within a predefined set of possible distributions, known as the uncertainty set. This approach contrasts with traditional empirical risk minimization (ERM), which focuses solely on average performance on training data. By minimizing the worst-case risk among an uncertainty set, distributionally robust models are better equipped to handle unseen distributional shifts. The uncertainty set can be defined in various ways, such as through statistical distances \citep{esfahani2017datadrivendistributionallyrobustoptimization}, moment constraints \citep{wiesemann2014distributionally, bertsimas2014datadrivenrobustoptimization, Hanasusanto_2018}, or causal assumptions \citep{rojascarulla2018invariantmodelscausaltransfer, rothenhäusler2020anchor}. This flexibility allows the framework to be tailored to specific application domains and types of distribution shifts. 

Causal models offer a principled way to define uncertainty sets based on the underlying data-generating process, rather than relying on postulated distances. Based on causal relationships, these models typically try to identify invariant features that remain stable in different environments \citep{peters2015causalinferenceusinginvariant, louizos2017causaleffectinferencedeep, pfister2021stabilizingvariableselectionregression}, providing a natural foundation for robustness. This data-driven approach avoids the need for ad hoc definitions of uncertainty sets, which may not capture the true nature of distributional shifts. Furthermore, causal frameworks enable the incorporation of domain knowledge, enhancing the interpretability and reliability of the resulting models. 
 
Most causality‐inspired robustness methods enforce invariance constraints, for instance, approaches such as \cite{arjovsky2020invariantriskminimization, christiansen2021causalframeworkdistributiongeneralization} aim to ensure that predictions remain invariant under arbitrarily strong perturbations or distribution shifts. In contrast, methods like anchor regression \citep{rothenhäusler2020anchor} and DRIG \citep{shen2023causalityoriented} offer \emph{finite-radius robustness} guarantees for an uncertainty set of finite-strength perturbations. However, existing methods that provide finite-radius guarantees have been confined to settings where the relationships among variables are linear, restricting their applicability to domains such as certain economic or physical systems. Yet, many modern applications, including image recognition and genomics, exhibit highly nonlinear interactions, and extending finite-radius robustness guarantees to such settings remains an open challenge.

To our best knowledge, we make the first attempt to develop a causality-inspired robustness method with a finite-radius robustness guarantee in nonlinear settings. This is achieved by integrating causal principles with modern techniques in representation learning, allowing us to handle nonlinear dependencies while maintaining robustness. This advancement opens up new possibilities for applying causal methods to a broader range of problems, including those involving high-dimensional and nonlinear data. Unlike traditional approaches such as adversarial training, and predefined distributional robustness optimization (DRO) frameworks, which often fail to generalize effectively in the presence of complex dependencies or multi-source heterogeneity, our method provides a unified solution that bridges DRO, causality, and nonlinear representation learning. This new proposal addresses the limitations of existing methods and offers a more comprehensive and flexible framework for robust prediction. Furthermore, our framework avoids relying on strong assumptions about causal identifiability or predefined robustness metrics, instead allowing the data itself to drive the learning process. This flexibility is particularly important in real-world scenarios, where the training data do not contain enough information to identify the underlying causal mechanisms, but its heterogeneity can be exploited to define the uncertainty set. 
We validate our theoretical results through experiments on both synthetic datasets and real-world single-cell data.

\section{Related work}

\subsection{Distributionally robust opimization} 

Despite significant advances in machine learning, the deployment of ML models in real-world applications often reveals their limitations and vulnerabilities. These challenges arise from distributional shifts \citep{mansour2023domainadaptationlearningbounds,rothenhäusler2020anchor, shen2023causalityoriented}, adversarial attacks \citep{goodfellow2015explainingharnessingadversarialexamples, madry2019deeplearningmodelsresistant}, and noisy or incomplete data. 

Traditional approaches to robustness, such as regularization or adversarial training, provide valuable insights but often struggle to generalize across different types of challenges. 
This limitation has driven interest in methods that aim to exploit the invariant properties of the data. \citet{arjovsky2020invariantriskminimization} propose Invariant Risk Minimization (IRM), a framework designed to encourage models to learn invariant predictors across diverse training environments. This paradigm is especially relevant when models are deployed in scenarios where distributional shifts are expected, as it focuses on isolating relationships that remain robust under varying conditions.
This led to an increase in interest in combining robust prediction schemes with other methodologies, such as representation learning and causality \citep{schölkopf2021causalrepresentationlearning}, or reinforcement learning and distributional robustness \citep{smirnova2019distributionallyrobustreinforcementlearning, lu2024distributionallyrobustreinforcementlearning}. 
One foundational work toward achieving robustness in a distributional sense is \citet{delage2010distributionally} who formalized DRO using moment-based uncertainty sets. A more recent approach considers DRO in terms of Wasserstein distance \citep{esfahani2017datadrivendistributionallyrobustoptimization, Hanasusanto_2018, kuhn2024wassersteindistributionallyrobustoptimization}. An interesting remark on the Wasserstein approach is noted in \citet{gao2020wassersteindistributionallyrobustoptimization}, essentially relating LASSO and several other estimators to solutions of DRO problems. However, though it enjoys certain theoretical guarantees, the approach considers the worst-case distribution contained within a region in Wasserstein distance, hence yielding overly conservative predictions. 
A different yet similar approach can be seen in \citet{popescu2005optimalmoment} or \citet{zhen2023unifiedtheoryrobustdistributionally}, where structural properties of distributions (e.g. symmetry, unimodality, and convexity) are integrated into an uncertainty set based on moments. \citet{wiesemann2014distributionally} generalizes moment-based ambiguity sets using conic inequalities, allowing for more flexible moment constraints beyond just first and second moments.
For further background in general DRO, we refer the reader to \citet{Rahimian_2022}. 

In contrast to classical DRO methods, where the set of perturbations against which the model is protected is prespecified and often overly conservative, recent causality-inspired frameworks focus on \textit{isolating} the relevant directions along which distributional shifts occur, in a data-driven way. A novel work by \citet{rothenhäusler2020anchor} introduces a causality inspired framework that guarantees robustness against shifts in mean. Building on this foundation, \citet{shen2023causalityoriented} proposed the \textit{Distributional Robustness via Invariant Gradients} (DRIG) framework, which extends these guarantees to both mean and variance. Both works exploit the heterogeneity of the data originating from multiple sources.
Other directions of theoretical studies in this direction include studying the objective of IRM, and its potential problems \citep{rosenfeld2021risksinvariantriskminimization}, as well as distributional matching guarantees in terms of environments needed \citep{chen2022iterative}. Further interesting results on online domain generalization are considered in \citet{rosenfeld2022online}. 

\subsection{Latent representation learning}
One of the central concerns in context of representation learning and dimensionality reduction is the question of identifiability --- whether the features or representations learned genuinely reflect the underlying data structure or are they not more than products of the chosen combination of hyperparameters. Research in this field has advanced in many interesting directions and has become increasingly productive in recent years \citep{khemakhem2020variational, khemakhem2020icebeemidentifiableconditionalenergybased,schölkopf2021causalrepresentationlearning,kivva2022identifiability,moran2022identifiabledeepgenerativemodels,wang2023posteriorcollapselatentvariable}. This leap forward features an examination of the assumptions regarding latent variables, the nature of their generative processes and distributions, and even of their decoder functions. Notable work by \citet{khemakhem2020variational} investigates the setting of conditionally independent latent variables given an observed auxiliary variable, through the iVAE framework. 
Other approaches have sought to provide identifiability guarantees using polynomial decoders \citep{ahuja2024interventionalcausalrepresentationlearning}, volume-preserving decoders \citep{yang2022volume}, sparse VAEs \citep{moran2022identifiabledeepgenerativemodels}, and about identifiability in general \citep{roeder2020linearidentifiabilitylearnedrepresentations, buchholz2023learninglinearcausalrepresentations}. However, despite all these impressive results, a breakthrough work by \citet{kivva2022identifiability} sets a new standard and challenges the necessity of auxiliary information in the latent structure for achieving strong guarantees for a broad class of functions. 
Most of these works conclude that identifiability of the hidden representation is possible only up to an affine transformation.
\citet{saengkyongam2023identifying} propose the Rep4Ex framework, which uses interventional heterogeneity to recover latent state representations up to an affine transformation, enabling reliable extrapolation to unseen interventions. Assuming a linear structural causal model with exogenous interventions and full-support residuals, Rep4Ex enforces “linear invariance” through a custom autoencoder objective, learning representations that generalize across observed and off-support actions. In contrast, our approach frames robustness to bounded shifts as a finite-radius DRO problem, providing closed-form certificates under noise assumptions.

\section{Method}
\subsection{Model setup} \label{Setting}
We observe a $d$-dimensional covariate $X$ and a real target variable $Y$. It is often reasonable to assume the data distribution is entailed by an underlying causal mechanism, which is much weaker than assuming the identifiability of such a causal mechanism. We would like to also account for nonlinear causal relationships. One way to incorporate nonlinearity is through a nonlinear representation map of the covariates $X$ that transforms a complex distribution of $X$ to a better-structured latent space where the causal relationship can be as simple as linear.

Specifically, we assume there exists a function $\phi^*: \mathbb{R}^d\rightarrow \mathbb{R}^k$ for $d\geq k$ such that the observed variables $(X,Y)$ follow a structural causal model (SCM) in the usual or observational setting: 
\begin{equation}\label{scm1}
\begin{pmatrix}
\phi^*(X)\\[\jot]
Y
\end{pmatrix} = \textbf{B} \begin{pmatrix}
 \phi^*(X)\\[\jot] Y
\end{pmatrix} + \varepsilon
\end{equation}
where $\varepsilon\in \R^{k+1}$ has zero expectation, and is allowed to have correlated components and \textbf{B} denotes the adjacency matrix of the causal graph. We assume that $\textbf{I}-\textbf{B}$ to be invertible, which is guaranteed if the graph is acyclic, where $\textbf{I}$ denotes the identity matrix. This model allows for nonlinear causal relationships between the covariates and between $X$ and $Y$, which can be represented as a linear SCM up to a nonlinear transformation of the covariates. 

In practice, we often encounter distribution shifts due to interventions on observed or latent variables. We consider a multi-environment setup as in \citet{shen2023causalityoriented}, where for each environment indexed by $e$, the distributions of the transformed covariates and the response are shifted by a random, additive intervention, i.e., 
\begin{equation}\label{scm2}
\begin{pmatrix}
\phi^*(X^e)\\[\jot]
Y^e
\end{pmatrix} = \textbf{B} \begin{pmatrix}
 \phi^*(X^e)\\[\jot] Y^e
\end{pmatrix} + \varepsilon + \delta^e
\end{equation}
where $\delta^e \in \R^{k+1}$, denoting the additive intervention, is independent of $\varepsilon$. We assume during training, we have access to multiple environments $\mathcal{E}$, where one of them, indexed by $0\in\mathcal{E}$, is the observational setting such that $\delta^0=0$, while the others are interventional settings, each with a distinct intervention variable $\delta^e$.

Furthermore, we are interested in out-of-distribution prediction, where the data may exhibit a different underlying distribution, namely, according to
\begin{equation}\label{scm3}
\begin{pmatrix}
\phi^*(X^v)\\[\jot]
Y^v
\end{pmatrix} = \textbf{B} \begin{pmatrix}
 \phi^*(X^v)\\[\jot] Y^v
\end{pmatrix} + \varepsilon + v
\end{equation}
where the intervention variable $v\in \R^{k+1}$ follows an unseen distribution different from that of $\delta^e$'s and is independent of $\varepsilon$, while the transformation $\phi^*$ and the graph structure $\textbf{B}$ stay the same. 

Our target is an optimal nonlinear prediction model that is robustness among distributions generated according to \eqref{scm3} for a class of new interventions. 
When $\phi^*$ is the identity map, \citet{shen2023causalityoriented} presented an approach to achieve an optimal linear model that is robust among certain test distributions. This motivates us to develop a two-step approach that consists of a first representation learning step to learn $\phi^*$ up to a certain non-identifiable equivalence class and a second step applying a robust prediction objective on top of the learned representations. 

\subsection{Two-step approach}
\subsubsection{Representation learning step}
\citet{shen2024distributional} proposed Distributional Principal Autoencoder (DPA) that learns low-dimensional representations while preserving the data distribution in the reconstructions.  
We build our representation learning step upon DPA, where the encoder maps from the data space to the latent space and the stochastic decoder maps from the latent space to the data space:
$$X \xrightarrow{\text{enc}(.)} {Z}  \xrightarrow{\text{dec}(.,\tilde\varepsilon)} \widehat{X}$$
where $\tilde\varepsilon$ follows the standard normal distribution. The original DPA ensures that the decoder produces reconstructions $\hat{X}$ that follows the same distribution as the original data $X$, and the encoder minimizes the unexplained variability in the conditional distribution of $X|\text{enc}(X)$. These are achieved by minimizing the following objective function jointly over the encoder and decoder: 
\begin{align*}
    L_{\text{DPA}} =\conditionaland \Ex_{X}\Ex_{\tilde{\varepsilon}} \| X-\text{dec}(\text{enc}(X),\tilde{\varepsilon})\| \conditionalnextrow \conditionaland - \frac{1}{2}\Ex_{X}\Ex_{\tilde{\varepsilon},\tilde{\varepsilon}'} \| \text{dec}(\text{enc}(X),\tilde{\varepsilon}) - \text{dec}(\text{enc}(X),\tilde{\varepsilon}')\|,
\end{align*}
where $\tilde{\varepsilon}, \tilde{\varepsilon}'$ are independently drawn from the standard normal distribution. For a fixed encoder, the objective function $L_{\text{DPA}}$ is the expected negative \emph{energy score} for the conditional distribution of $X|\text{enc}(X)$; see \citet{shen2024distributional} for details. 

Here, 
we account for heterogeneity across different environments by encouraging the learned representations from the encoder to follow a mixture of Gaussians. 
To ensure this, a third neural network, called the \textit{prior} network $g$, is introduced in addition to the standard encoder-decoder framework already present. The prior network takes the environment labels $E$ as input and produces a sample of the latent vector ${Z}_g$ from a mixture of Gaussians; specifically, $g(E,\xi)=\mu_g(E) + \xi\Sigma_g(E)$ for $\xi$ standard normal. 
$$E \xrightarrow{g(., \xi)} {Z}_g. $$
Furthermore, since we want to encourage the latent to emulate a mixture of Gaussians, we also augment the loss function of the DPA to enforce the encoder output matches the distribution of the prior. The new loss term is the negative energy score for the conditional distributions of $\text{enc}(X)|E$:
\begin{align*}
    L_G =\conditionaland \Ex_{X,E}\Ex_{\xi} \| \text{enc}(X)- g(E,\xi)\|  \conditionalnextrow \conditionaland - \frac{1}{2}\Ex_{E}\Ex_{\xi,\xi'}\|g(E,\xi) - g(E,\xi')\|
\end{align*}
where $\xi,\xi'$ are independently drawn from a standard normal distribution. The formulation can be thought of as a conditional version of the DPA. 

Indeed, just as the optimum  of the loss function $L_{\text{DPA}}$ in \citet{shen2024distributional} was motivated by the goal of ensuring $$ \text{dec}^*(z,\tilde\varepsilon) \enspace \enspace= \enspace \enspace X|\{ \text{enc}^*(X)=z\}, \enspace \enspace \text{in distribution} \enspace \forall z, $$ the augmentation loss function $L_G$ was inspired by the goal of achieving 
$$g^*(e, \xi) \enspace \enspace=\enspace \enspace\text{enc}^* (X) | \{E=e\},  \enspace \enspace \text{in distribution} \enspace \forall e$$ where $\text{enc}^*,\text{dec}^*$, $g^*$ denote the optimized encoder, decoder and prior network, respectively.
The final augmented loss function reads 
\begin{align}
    L_{\textrm{RL}} = L_{\text{DPA}} + \alpha L_G, \label{finallossfct}
\end{align} for a selected hyperparameter $\alpha$. We define 
\begin{equation*}
	(\text{enc}^*, \text{dec}^*) \in \argmin_{(\text{enc}, \text{dec})} L_{\textrm{RL}},
\end{equation*}
where we also sometimes refer to $\text{enc}^*$ as $\widehat{\phi}$. 
Moreover, optimizing $L_{\textrm{RL}}$ gives $\widehat{\phi}(X) = \widehat{Z}$. Since DPA learns the distribution of $X$, as well as the distribution of its principal components in the latent space, the estimated latent vector is an affine transformation of the true latent vector $\widehat{z}=Az+c$, for an invertible matrix $A$ and a vector $c$. 
Affine identifiability is guaranteed in a range of settings, and it can be achieved by imposing conditions on the distribution of latent, variables (for example, presence of interventions), or on the mixing function. 
Since the autoencoder scheme learns to match the distribution of $X$, affine identification is ensured by \hyperref[Kivvalemma1]{Lemma~\ref*{Kivvalemma1}} from \citet{kivva2022identifiability}.  
It is worth noting that there are also other similar results which can be used to ensure affine identification in this setup. For example, since interventions are present in the considered setting, \hyperref[ahujaresult]{Lemma~\ref*{ahujaresult}} from \citet{ahuja2024interventionalcausalrepresentationlearning} can also be applied using VAE \citep{kingma2022autoencoding}. Some known applicable results on identifiability can be found in Appendix \ref{idfby}. 

\subsubsection{Causality-inspired robustness step}
Motivated by its guarantees for distributional robustness in linear settings and its adjustable robustness radius parameter, $\gamma\geq 0$, we employ the DRIG method \citep{shen2023causalityoriented} to learn a robust linear model on top of the representations learned in the first step.
Specifically, let $\widehat{Z}^e=\widehat\phi(X^e)$ be the learned representations for each environment $e\in\mathcal{E}$. Additionally, DRIG requires that all environments be centered relative to the mean of the reference environment, for both $\widehat{Z}^e,Y^e$. To satisfy this requirement, we adopt the following centering step; $$\widehat{z}^e_c = \widehat{z}^e - \mathbb{E}[\widehat{z}^0] = A(z^e - \mathbb{E}[z^0]) = A z^e_c, \enspace y^e_c = y^e - \mathbb{E}[y^0]$$ where this operation constitutes a linear transformation of the true latent variable. We then define the linear coefficients by
\begin{align}
    \widehat{b}_\gamma = \argmin_b L^\gamma_{\textrm{CIR}}(b),
\end{align} for
\begin{align*}
    L^\gamma_{\textrm{CIR}}(\conditionaland b) = \Ex [Y^0_c - b^\top \widehat{Z}^0_c]^2 \conditionalnextrow
    \conditionaland + \gamma \sum_{e\in\E} \omega^e \bigg(\Ex [Y^e_c - b^\top \widehat{Z}^e_c]^2 - \Ex [Y^0_c - b^\top \widehat{Z}^0_c]^2 \bigg),
\end{align*}
where $\omega^e>0$ denote environment weights and it holds $\sum_{e\in \E} \omega^e = 1$. For example, in the uniform case, $\omega=1/|\E|$, or $\omega^e=n_e/n$, where $n_e$ denotes the number of samples from environment $e$. 
We remark that $\widehat{b}_\gamma$ can be computed explicitly, by solving the equations obtained by first order conditions; 
\begin{align}\label{drigloss}
    \widehat{b}_\gamma = \bigg((1-\gamma)\,{\widehat{Z}^{0}_c}{}^{\top} \widehat{Z}^{0}_c + \gamma \sum_{e\in \mathcal{E}} \omega^e {\widehat{Z}^{e}_c}{}^{\top} \widehat{Z}^{e}_c \bigg)^{-1}\bigg( (1-\gamma){\widehat{Z}^{0}_c}{}^{\top} Y^{0}_c + \gamma \sum_{e\in \mathcal{E}} \omega^e {\widehat{Z}^{e}_c}{}^{\top} Y^{0}_c  \bigg).     
\end{align} 

\subsubsection{Final algorithm}
Building upon the previous two sections, we now summarize our two-step algorithm as follows:
\begin{itemize}
    \item Learn $\widehat{\phi}$ by optimizing $L_{\textrm{RL}}$ in equation \ref{finallossfct}.
    \item Estimate $\widehat{b}$ from $(\widehat{Z}_c, Y_c)$ as the solution of equation \ref{drigloss}.
    \item Define the final prediction model as $\widehat{f}(x) = \widehat{b}^\top\widehat\phi_c(x)$.
\end{itemize} 
We call our method CIRRL (Causality-Inspired Robustness via Representation Learning).

\subsection{Theoretical guarantees}
This subsection formalizes the theoretical guarantees for robustness and identifiability. Proofs of all the results can be found in Appendix \ref{proofsapx}. 
To quantify the robustness of a prediction model $f$, we use the following worst-case risk:
\begin{equation*}
    \mathcal{L}_\gamma(f) = \sup_{v\in C^\gamma} \Ex_v [Y - f(X)]^2,
\end{equation*}
where
\begin{equation*}
    C^{\gamma} = \bigg\{ v\in \R^{k+1} \enspace | \enspace \Ex[vv^\top] \preceq S^0 + \gamma \sum_{e\in \E} \omega^e \bigg( S^e - S^0 + \mu^e {\mu^e}^\top \bigg) \bigg\}
\end{equation*}
with $S^e = \cov[\delta^e]$ and $\mu^e = \Ex [\delta^e]$. 
In the case where $f$ can be rewritten as $f(x)=b^\top\phi(x)$ for some vector $b$ and an arbitrary function $\phi$, the worst-case risk $\mathcal{L}_\gamma({f})$ can be rewritten as $\mathcal{L}_\gamma(b^\top \phi) = \sup_{v\in C^\gamma} \Ex_v [Y - {b}^\top {\phi}(X)]^2$. 
The following lemma characterizes the set of perturbations against which the model is robust, in terms of $v$ and its distribution. 
Let $\phi_c(X):= \phi(X) - \Ex \phi(X^0)$ denote the centered version of $\phi$. 
\begin{proposition}\label{driglemma}
Assume that the SCMs (Equations \ref{scm2}, \ref{scm3}) hold as described above. Let $\phi: \R^d \rightarrow \R^k $ be an affine transform of $\phi^*$, i.e. $\phi(x)=N\phi^*(x)+m$  
for an invertible $k\times k$ matrix $N$, and a vector $m$. Then, the loss function $\mathcal{L}_\gamma(b^\top\phi_c)$ has explicit form, namely 
\begin{equation*}
    \mathcal{L}_\gamma(b^\top \phi_c) \conditionaland = \Ex [Y^0 - b^\top \phi_c(X^0)]^2 \conditionalnextrow \conditionaland + \gamma \sum_{e\in\E} \omega^e \bigg(\Ex [Y^e - b^\top \phi_c(X^e)]^2 - \Ex [Y^0 - b^\top \phi_c(X^0)]^2 \bigg),
\end{equation*}
where $\omega^e>0$ denote environment weights such that $\sum_{e\in \E} \omega^e = 1$. 
\end{proposition}
This proposition provides a characterization of robustness guarantees, showing that the model can tolerate perturbations $v$ whose second moments are bounded by a weighted combination of second moments of training environments. Importantly, the degree of robustness can be controlled through the hyperparameter $\gamma$. 
Under relatively mild conditions (see Appendix \ref{idfby}), the latent variables can be identified up to an affine transformation. For example, for a piecewise affine decoder function $f$,  
\begin{lemma}\label{Kivvalemma1} 
    [\cite{kivva2022identifiability}] Let $\dim(E) = 1$. Under the Gaussian mixture model (GMM) as defined in Appendix \ref{idfby} and for $f$ a piecewise affine decoder function, if $f$ is weakly injective (see Appendix \ref{idfby}), $P(E,Z)$ can be identified from $P(X)$ up to an affine transform.
\end{lemma}
Since the DPA part matches the distribution of $\widehat{X}$ to the distribution of $X$ when optimized we obtain an affine transform of the true latents. This provides a crucial foundation for designing a method that leverages a linear estimator on top, to obtain predictions. 
\begin{assumption}\label{assumptionintercept}
    $\mathbb{E}[v]=\Sigma [(\I - \Bb)^{-1}_{1:k, \cdot}]^\top\alpha$ for some $\alpha\in \R^k$, where $\Sigma=\mathrm{Cov}[\varepsilon + v]$.
\end{assumption}
This assumption means that \textit{the average effect of $v$ on $Y$ in the SCM passes only through $Z$}. This is easy to see, as the resulting is a vector of linear combinations of rows of $(\I - \Bb)^{-1}_{1:k, \cdot}$, which corresponds to the total causal effect on the vector $Z$. 
Based on the above results, the following theorem establishes the central theoretical finding of this work. 
It indicates that our learned prediction model is the most robust among all square-integrable functions. 
\begin{theorem}\label{nonlinrob}
Under SCMs described in equations \ref{scm2}, \ref{scm3}, and Assumption \ref{assumptionintercept} assume that $\varepsilon,v$ are elliptical (Definition \autoref{spherical_elliptical}) and recall that $\varepsilon$ and $v$ are independent. Let $\mathcal{X}$ denote the unbounded support of $X$ and $\Le$ the usual space of square-integrable functions over $\Xx$. 
Then,
$$\mathcal{L}_\gamma(\widehat{f}) = \min_{f\in \Le} \mathcal{L}_\gamma(f) $$
for $\widehat{f}(X) = \widehat{b}^\top \widehat{\phi}_c(X)$. 
\end{theorem}
\begin{remark}
    Examples of elliptical distributions include multivariate Gaussian and multivariate $t$-distributions. 
\end{remark}

\section{Experiments} \label{exps}
We validate our theoretical results in a simulated environment and a real dataset involving gene perturbations. For context, we compare our method to other approaches that can be adopted in the setting concerning nonlinear and robust prediction.
Invariant risk minimization, parameterized by a fully connected neural network as introduced in \citet{arjovsky2020invariantriskminimization}. As a robust optimization method, IRM considers the heterogeneity of the data, protecting against arbitrarily strong perturbations in the distribution.
Empirical risk minimization, also parameterized by a fully connected neural network. This approach ignores heterogeneity in the data and aims only to find the best fit on the training set, in terms of MSE. A single (L4) GPU was used to train all models, and Adam \citep{kingma2017adammethodstochasticoptimization} was used for optimization. For details of the hyperparameters used in optimization, we refer to Appendix \ref{hyperparams}. 

\subsection{Synthetically generated dataset}
In simulated experiments, the data was generated according to the SCM introduced in Subsection \ref{Setting}, and the technical details are described in Appendix \ref{generatingscheme}. We sample one observational and four interventional environments with 2000 samples, each, from a two-dimensional latent variable model and embed them in 10 dimensional space using a polynomial decoder function. Out-of-distribution (OOD), or the test set, is generated with a fixed perturbation strength (Appendix \ref{generatingscheme}), with latents drawn from a Gaussian and $\chi^2$ for the well-specified and misspecified case, respectively. 

In realistic settings, the number of latent variables is unknown; and a criterion is required to determine the optimal latent dimension of the model. We refer to Figure \ref{fig:elbowplot_sc}, and, as shown, performance increases significantly when changing from one to two hidden dimensions, after which the performance stabilizes. 

We also evaluate the model in a misspecified OOD scenario disregarding the distributional assumption in \autoref{nonlinrob}. Specifically, we generate the OOD set from a $\chi^2$-distribution its degrees of freedom approximately corresponding to $1-$norm of $v$. The rest of the experimental setup mirrors that of the scenario im the well-specified case, with the only difference being the mentioned distribution of the test set. The results show no indication of significantly degraded performance compared to the ERM and IRM approaches. These findings suggest that the elliptical assumption in \autoref{nonlinrob} is not a fundamental constraint and could probably be relaxed without damaging the main guarantees.

\subsection{Single-cell datatset}

We evaluated the method using a large single-cell RNA sequencing dataset \cite{replogle2022singlecelldata}. This dataset involves genome-wide CRISPR-based perturbations performed on millions of human cells to systematically target expressed genes. For our analysis, we focus specifically on the subset of data derived from RPE1 cells, which emphasize genes likely to play critical roles and exhibit responsiveness to interventions. Following the preprocessing steps established by \citet{chevalley2023causalbenchlargescalebenchmarknetwork}, we selected the 10 genes with the highest expression levels to serve as observed variables. Among these, one gene is treated as the response variable, while the other nine are treated as predictors, based on the reasoning provided by \citet{shen2023causalityoriented}. Our training set consists of 11,485 reference samples, supplemented by data from 10 distinct interventional environments, each corresponding to a targeted perturbation of one of the observed genes. The number of samples per intervention varies between 100 and 500. Furthermore, the dataset includes numerous additional environments, where interventions are applied to unobserved genes outside the set of 10 selected genes. These unseen environments, which differ from those in the training set, are used as test scenarios to evaluate the robustness of prediction models across diverse distributions.

As is evident in Figure \ref{fig:elbowplot_sc}, there is no obvious latent dimension suggested in the results; therefore, since two is the most pronounced point, we select one point to the right. Figure \ref{fig:testenvsplotqq} and Figure \ref{fig:mainhistogram} show a competitive performance of the proposed model in unseen perturbations. In addition, the performance metric shows that robustness stabilizes once $\gamma$ reaches a moderate value. For instance, when $\gamma >10$ the performance levels off, indicating that fine-tuning $\gamma$ is not as critical for achieving better robust performance compared to standard approaches like IRM or ERM. 
We also analyze the performance of the approaches within identical environments. In Figure \ref{fig:mainhistogram}, the boxplots illustrate the discrepancies between the mean squared errors of a competitive method and CIRRL across various environments. Notably, CIRRL generally achieves superior prediction accuracy, particularly when $\gamma$ is set to a higher value.

\section{Conclusion}
In this work, we proposed a robust framework for estimating target variables in high-dimensional settings, designed to handle distribution shifts. We demonstrated its effectiveness on both synthetic and real-world datasets, consistently outperforming traditional methods like ERM and IRM. Our synthetic experiments confirmed the model's resilience to distributional variations, including deviations from ellipticity, suggesting potential for broader applications without strict structural assumptions. Future work could explore more complex latent structures, interactions between observed and unobserved variables, and relax the ellipticity assumption. Extensions to non-independent data or sequential environments also present promising directions.
\begin{figure*}[h]
\centering
    \includegraphics[width=\textwidth, height=4.7cm]{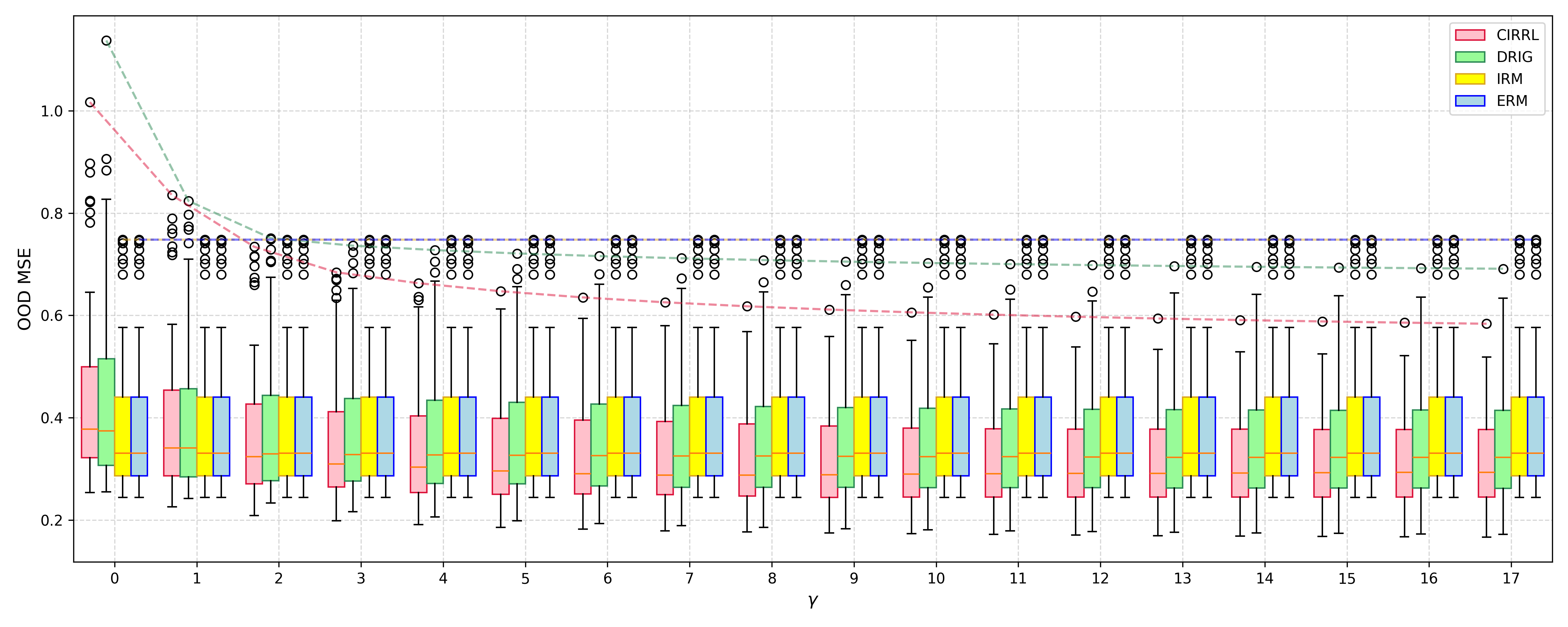}
\caption{Single-cell dataset.
Figure presents boxplots of the MSE across environments. In each group of four boxes—each corresponding to a specific value of $\gamma$ (as indicated on the $x$-axis).
Note that both IRM and ERM do not depend on $\gamma$, which is why their boxplots remain unchanged across different groups. The dotted lines overlaid on each group indicate the worst performing environment, marking the maximum MSE error (or worst quantile) observed.
}\label{fig:mainhistogram}
\end{figure*}
\begin{figure*}[h]
\centering
    \includegraphics[width=0.49\textwidth, height=4.2cm]{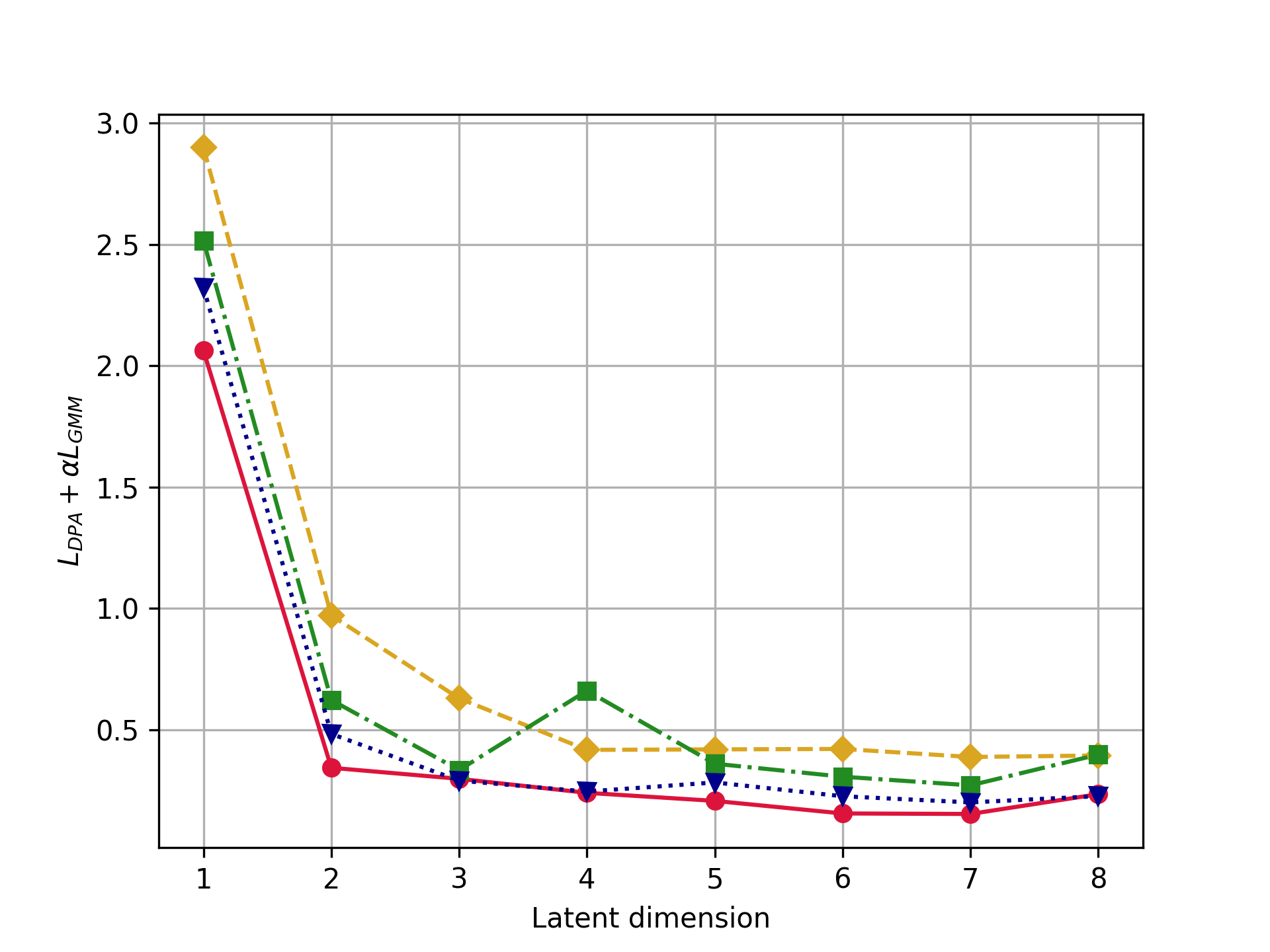}
    \includegraphics[width=0.49\textwidth, height=4.2cm]{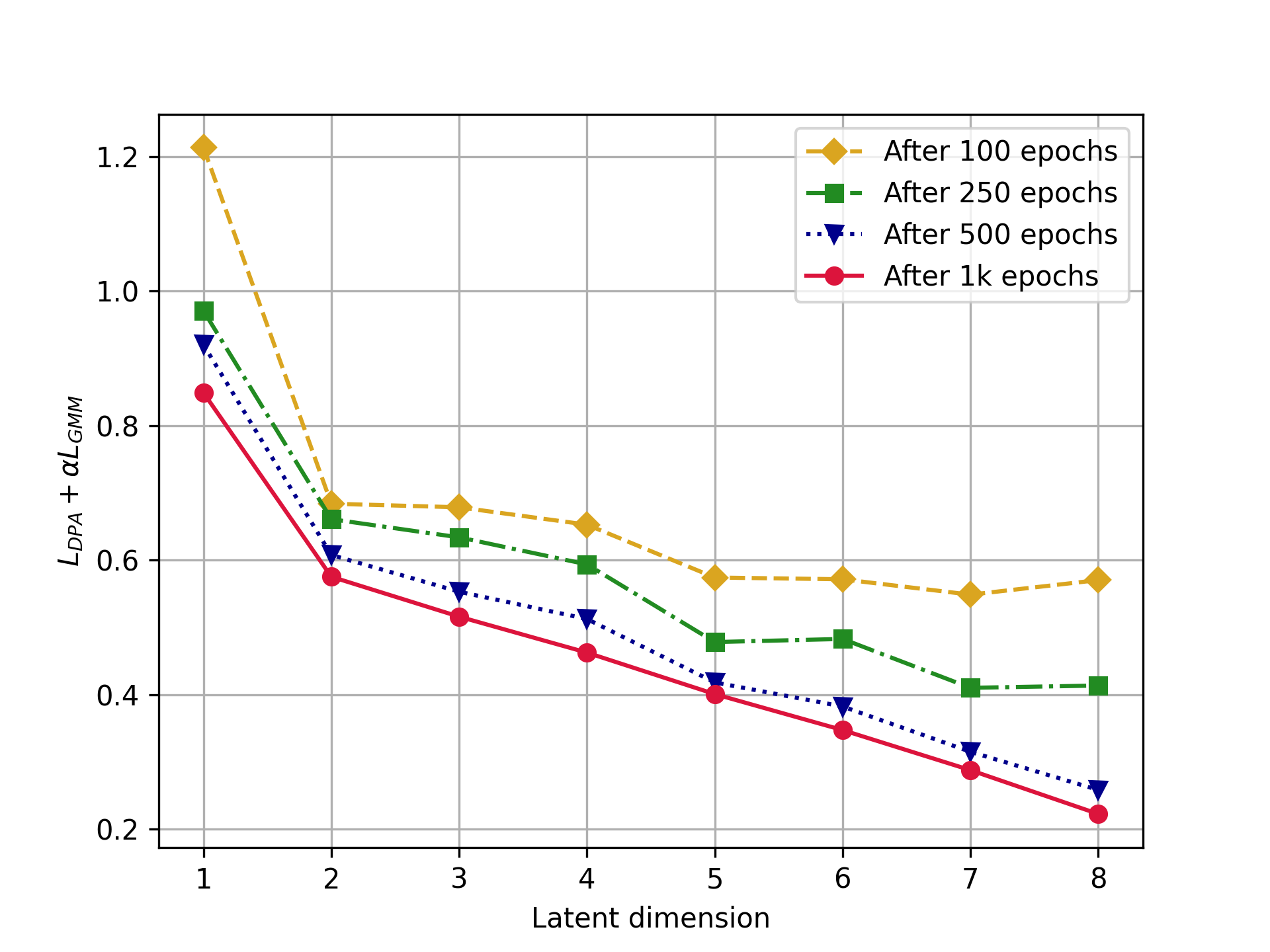}
\caption{Simulated (left), and single-cell (right) dataset - values of optimized loss function $L$ (equation \ref{finallossfct}). In case there is a \emph{clear} elbow point, as in the left figure, one should choose it as the latent dimension for the model. However, if the elbow point is less pronounced as it is the case in the second figure, we recommend choosing a value slightly to its right, in this example three or four.}\label{fig:elbowplot_sc}
\end{figure*}
\begin{figure}[h!]
\centering
    \includegraphics[width=0.49\textwidth, height=4.2cm]{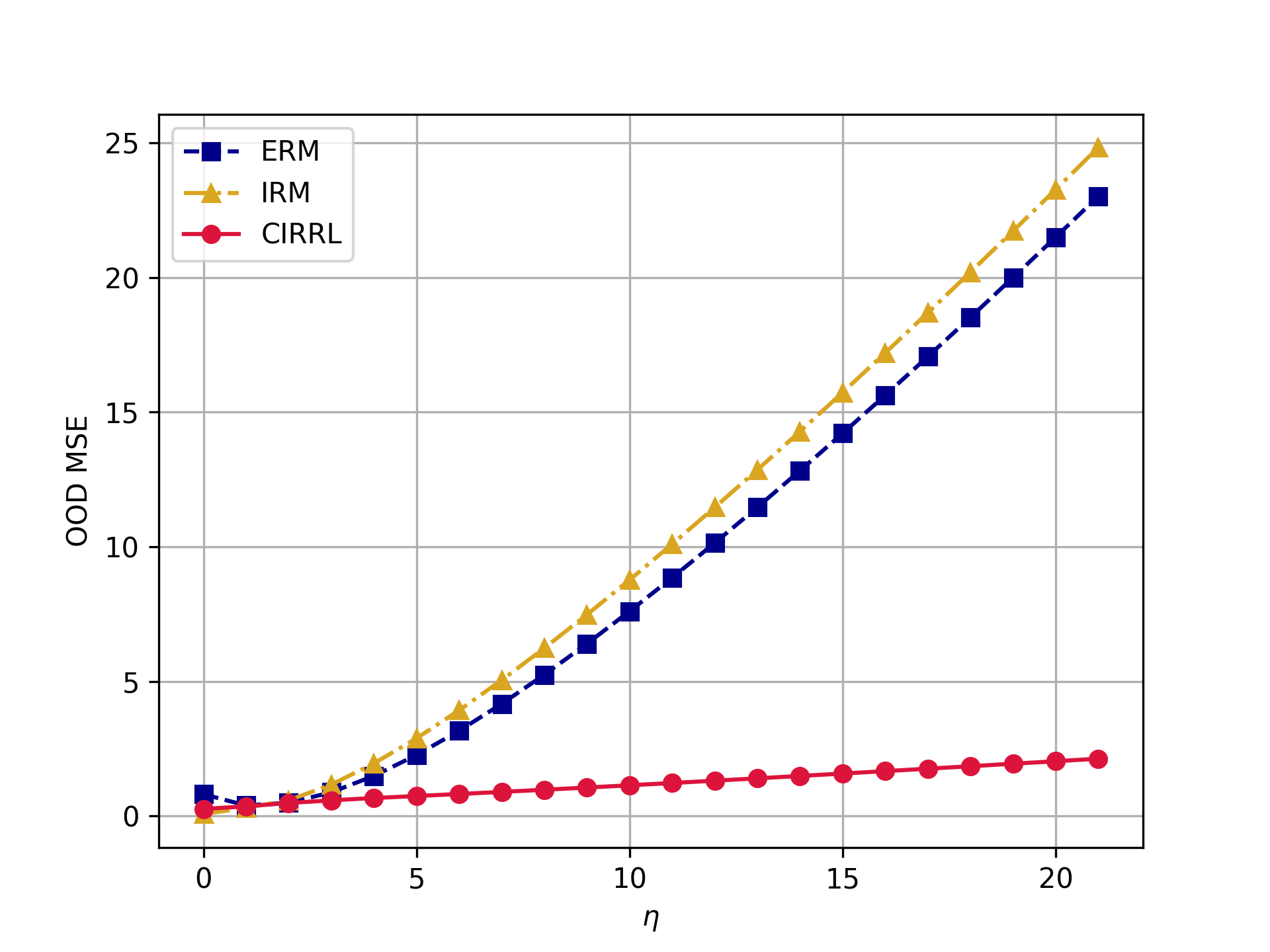}
    \includegraphics[width=0.49\textwidth, height=4.2cm]{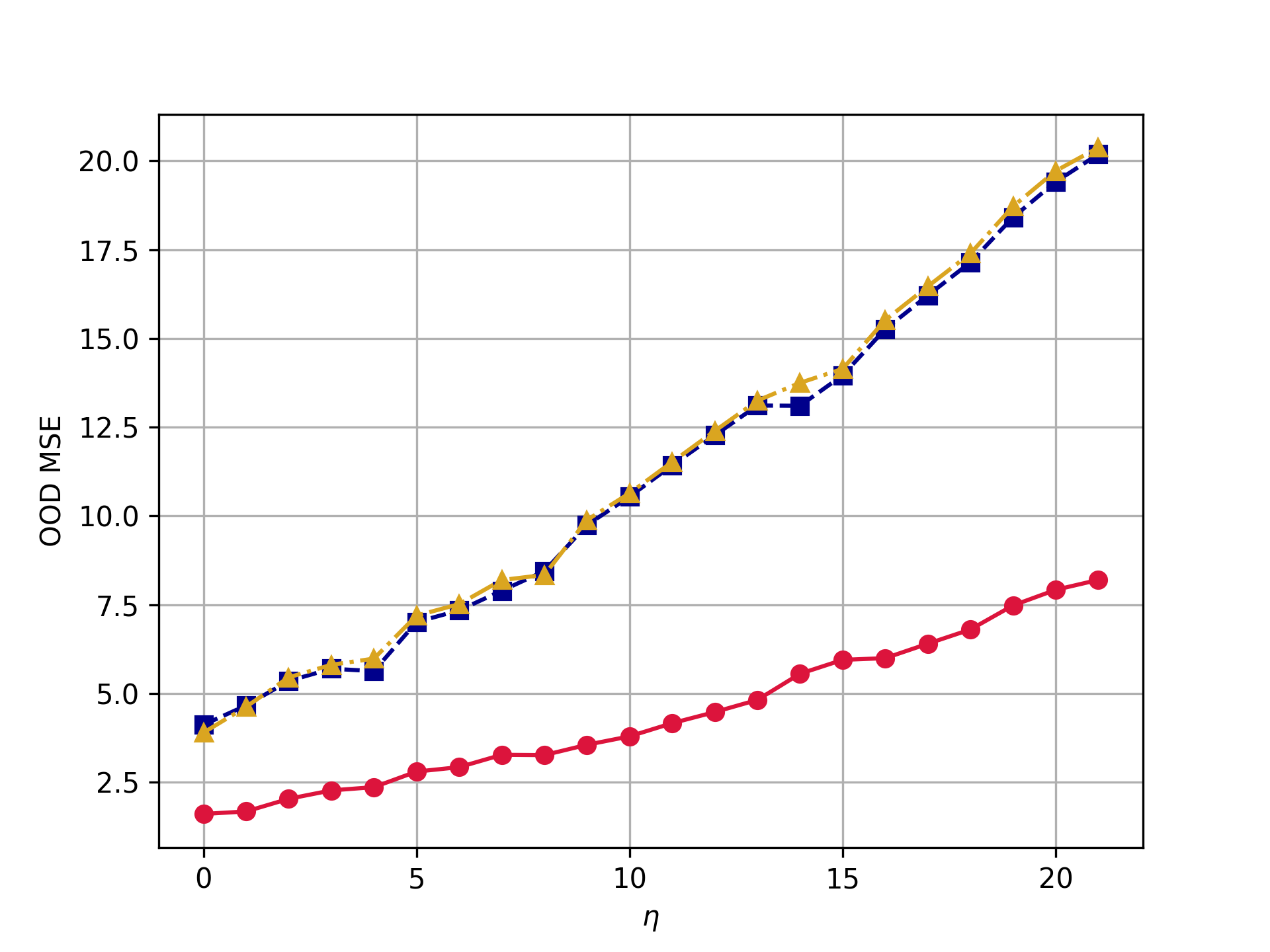}
\caption{Synthetic dataset for well-specified (left) \& misspecified setting (right)- final OOD MSEs for IRM, CIRRL, and ERM for various perturbation strengths $\eta$ (Appendix \ref{generatingscheme}). The inferiority of IRM to ERM could be attributed to the \emph{finite} nature of the perturbations.
}\label{fig:synth_lossvseta}
\end{figure}

%% file: Appx.tex
\newpage
\appendix
\onecolumn

\section{Symbols}
\begin{table}[h]
\caption{Table of Notation}
\label{notationstable}
\vskip 0.15in
\begin{center}
\begin{small}
\begin{sc}
\begin{tabular}{lcccr}
\toprule
Symbol & Meaning \\
\midrule
$\textbf{B}\in \R^{(k+1)\times(k+1)}$ & Adjacency matrix of the latent causal graph \\
$\textbf{I}$ &  Identity \\
$\phi^*$ & True decoder function \\
$X, X^e$ & Observed covariates, observed covariates in environment $e$, respectively \\
$Z, Z^e$ & Latent variables, latent variables in environment $e$, respectively \\
$Y, Y^e$ & Response, response in environment $e$, respectively \\
$\varepsilon \in \R^{k+1}$ & Noise variable \\
$\xi, \tilde{\varepsilon}$ & Noise variables for $g, \text{dec}$, respectively \\
$\delta^e \in \R^{k+1}$ & Intervention in environment $e$ \\
$v\in \R^{k+1}$ & Intervention in the test environment \\
$\text{enc}^*, \widehat{\phi}$ & Optimized encoder function \\
$\text{dec}^*$ & Optimized decoder function \\
$g^*$ & Learned and optimized prior encoder function \\
$\mu_g, \Sigma_g$ & $g$ estimates of the latent mean and covariance \\
$\widehat{b}$ & Estimated DRIG coefficient \\
$\gamma$ & Perturbation strength to account for in DRIG estimation \\
$C^{\gamma}$ & Uncertainty set of DRIG \\
$\widehat{f}$ & Final model \\
$\mathbb{L}^2(\mathcal{A})$ & Space of square integrable functions over $\mathcal{A}$\\
$\eta$ & Perturbation strength in the test environment - simulated data \\
\bottomrule
\end{tabular}
\end{sc}
\end{small}
\end{center}
\vskip -0.1in
\end{table}

\section{Identifiability}\label{idfby}
\begin{definition} (Weakly injective)
    Let $f:\mathbb{R}^k\rightarrow \mathbb{R}^d$ and $k\leq d$
    . A function $f$ is said to be weakly injective, if 
    \begin{itemize}
        \item There is a point $x_0\in \mathbb{R}^d$ and $\delta>0$ with $|f^{-1}(x)| = 1 $ for all $x$ in the image of $f$ intersected with the $\delta$-ball around $x_0$, \textbf{and}
        \item the set $\{ x\in \mathbb{R}^d : |f^{-1}(x)|=\infty \}\subseteq f(\mathbb{R}^k)$ has Lebesgue measure zero.
    \end{itemize}
\end{definition}
Assuming $P$ follows a Gaussian mixture, 
together, $E$ with $Z$ encodes the latent structure
$$E = e \sim P_\lambda(E=e) $$ 
$$Z | E=e \sim \mathcal{N}(\mu_e, \Sigma_e) $$
$$E \rightarrow Z \rightarrow X$$
where $P_{\lambda}$ denotes the marginal distribution of $E$ that depends on $\lambda_j$, and $\sum_{j=1}^J \lambda_j = 1$ denote the positive weights of the components in the Gaussian mixture.
For multivariate but discrete $E$, one is required to reconstruct everything just from the distribution of $X$. Under stronger conditions, it is also possible to identify $P(E)$.
\begin{lemma} \label{Kivvalemma2} 
    \citep{kivva2022identifiability} Let the GMM model as defined above hold and let $f$ be piecewise affine. If $f$ is weakly injective, $P(Z)$ is identifiable from $P(X)$ up to an affine transform. 
\end{lemma}
\cite{ahuja2024interventionalcausalrepresentationlearning} assume that the interior of the support of $z$ 
is nonempty subset of $\mathbb{R}^k$, and that the decoder function $f$ is an injective polynomial of finite degree. Also, assuming that the interior of the encoder image is a nonempty subset of $\mathbb{R}^k$ 
\begin{lemma}\label{ahujaresult}
    If the autoencoder solves the reconstruction identity $\text{dec}\circ\text{enc}(x)=x$ for all $x$ in its support, under the constraint that the learned decoder is a polynomial of the same degree as the true decoder, then it achieves affine identification, ie recovers an affine transformation of the true latent variables.
\end{lemma}
For a deeper overview of similar and stronger results, concerning even more general decoders, we refer the reader to \citet{khemakhem2020icebeemidentifiableconditionalenergybased, kivva2022identifiability, ahuja2024interventionalcausalrepresentationlearning, buchholz2023learninglinearcausalrepresentations}.  

\section{Proofs}\label{proofsapx}
\input{Appendix_newproofs.tex}

\newpage

\section{Hyperparameters} \label{hyperparams}
The architecture of the ERM, IRM networks, including depth, width, and hyperparameters such as batch normalization, learning rate, dropout, and the number of training epochs, was designed to match the corresponding networks in DPA. 
During training, the models were fed batches of $(X,Y)$ pairs, and the network predicted $Y$ directly from $X$. Hyperparameters that have not been changing: width was generally taken to be 400 units in all networks as well as a batch norm. In the DPA all networks had the same depth of two in experiments with synthetic datasets and four for studies on single-cell data. Dropout is also included where it boosts performance, in these cases, only within the ERM, with $p=0.25$ and 0.5 for synthetic and single-cell studies, respectively. All experiments used Adam for optimization, with a learning rate $10^{-4}$ and $\alpha=10^{-1}$ . The latent dimension was chosen considering a point of type \emph{elbow}, after which there was no significant performance gain in terms of the objective function $L$. 
Most experiments do not require significant memory to run as the models are essentially common neural networks and using a GPU usually take between 7 and 12 minutes for 1000 epochs of training. 

\section{Synthetically generated Data}\label{generatingscheme}
It is often helpful and insightful to test new methods and theories in controlled settings, for example, in simulated or lab environments, to ensure that everything works as expected. Let $X\in \mathbb{R}^d$, let $Z\in \mathbb{R}^k$, and let $Y$ be a real variable. As before, it is implicitly assumed that $k\leq d$, as well as $d\geq 2$. Let $\phi^*: \mathbb{R}^d\rightarrow \mathbb{R}^k$ denote the true encoder function. The aim is to generate data from environments according to the causal graph similar to the one in subsection \ref{Setting}. The true encoder function in synthetic experiments is chosen as a polynomial, parametrized through a random coefficient matrix which adheres to the constraints described in \citet{ahuja2024interventionalcausalrepresentationlearning} section 4. 

That is, $X$ admits a lower dimensional representation $Z = \phi^*(X)$, which is a nonlinear function of $X$, while $Y$ itself is a linear function of $Z$. For an environment $e\in\E$, the corresponding modified version of the structural equations as noted in equation \ref{scm2} is
\begin{equation}
\begin{pmatrix}
\phi^*(X^e)\\[\jot]
Y^e
\end{pmatrix} = \textbf{B} \begin{pmatrix}
 \phi^*(X^e)\\[\jot] Y^e
\end{pmatrix} + \varepsilon + \delta^e
\end{equation}
where $e\in \E$ and $e=0$ denotes the observational environment and $\delta^e$ denotes the intervention which is independent of $\varepsilon$. Analogously, the test setting, for $v$ also independent of $\varepsilon$, as in equation \ref{scm3}. 
\begin{equation}
\begin{pmatrix}
\phi^*(X^v)\\[\jot]
Y^v
\end{pmatrix} = \textbf{B} \begin{pmatrix}
 \phi^*(X^v)\\[\jot] Y^v
\end{pmatrix} + \varepsilon + v
\end{equation}
To begin, a random directed graph $G$ is generated, with edge probabilities of $\frac{1}{2}$. To enforce the acyclic condition, the edges of the random graph $G$ are filtered to eliminate cycles. In particular, the new graph will only include edges where the source node has a \textit{higher} index than the target node. The remaining edges are then assigned a weight drawn from the normal distribution. Let $\mathbf{B}$ denote the adjacency matrix of this directed acyclic graph (DAG) as above, and let $\textbf{C}$ denote its inverse $(\mathbf{I} - \mathbf{B})^{-1}$. Let $b$ denote the first $k$ components of the $(k+1)$-st row, $\mathbf{B}_{k+1,1:k}$. 
\\
Next, random positive definite matrices with norm 1 are generated to serve as the covariance matrices $\Sigma_e$ of the interventions $\delta^e$ in each environment, as well as normalised random vectors $\mu_e$ to be used as their means. The second moment of the test intervention $v$ is then derived from the covariance matrices and means of the individual interventions 
$$\Xi_{\eta} = \frac{\eta}{|\E|} \sum_{e\in\E} \Big( \Sigma_e + \mu_e \mu_e^\top\Big)$$
and is controlled by the parameter $\eta$, which is given as input. The covariance matrix of $\varepsilon$ is generated similarly to those of $\delta^e$. The role of $\eta$ is to control the strength of perturbation within the test set. 
\\
Furthermore, the means and covariance matrices are used to sample $\varepsilon$ as well as $ \delta^e$ from the normal distribution, for each environment $e$. The data points are subsequently generated as 
\begin{equation}
\begin{pmatrix}
Z^e\\[\jot]
Y^e
\end{pmatrix} = \textbf{C} (\varepsilon + \delta^e).
\end{equation}
Lastly, $X^e$ is obtained by plugging $Z^e$ into the specified latent function. The latent function can be taken as a polynomial of given degree, with some constraints on the dimensionalities of $X$ or $Z$ as mentioned in Ahuja et al. \citet{ahuja2024interventionalcausalrepresentationlearning}. 
\\
Alternatively, the latent function can also be chosen as an initialised ReLU network of given width and depth to cosplay as a nonlinear function. Training data points are collected, as well as their labeled environments encoded in $E$, and are divided in batches to be used in training. Concerning the data for the test environment, $\varepsilon$ is sampled analogously as in the training setting, but $v$ is sampled as a Gaussian with mean $\mu_v$ and covariance $\Xi_{\eta} - \mu_v \mu_v ^\top$
where $\mu_v$ can be given as input or it defaults to $\mu_v = \frac{\eta}{|\E|} \sum_{e\in \E} \mu_e$. Also similarly to the train setting, the data points are subsequently generated as
\begin{equation}\label{distr_node}
\begin{pmatrix}
Z^v\\[\jot]
Y^v
\end{pmatrix} =\textbf{C} (\varepsilon + v).
\end{equation}
and $X^v$ is obtained after plugging $Z^v$ into the same latent function. 
Furthermore, in case of student's $t$-distribution, the vector $\varepsilon+v$ is derived with one additional step before multiplication by $\textbf{C}$. Taking $\varepsilon+v$ generated as above, denote it by $\zeta$. Since $\zeta$ is Gaussian, the vector $$\zeta':= \zeta_c \cdot \sqrt{\frac{\nu}{u}} + \text{mean}_\zeta, \enspace u\sim \chi_\nu^2,$$ where $\zeta_c$ is simply $\zeta$ centered, follows a multivariate student's $t$-distribution with $\nu$ degrees of freedom. By abuse of notation, denote $\zeta'$ by $\varepsilon+v$ and plug it into equation \ref{distr_node}.

Additionally, since the same number of observations is drawn from each environment, the weights $\omega^e$ simplify to the uniform weights in this context. There is also an option to consider the setting where one excludes any intervention on the target variable $Y$, this would mean that the means and covariance matrices of $\delta^e,v$ are generated with a last entry, respectively, row being zero, and the matrices are in that case only positive semidefinite.

The submitted code also contains an oracle estimator, namely the MSE of DRIG applied on the true latent variables, which is considered to be the theoretically optimal achievable DRIGs MSE in case of perfect reconstruction. 
Furthermore, the experimental data was drawn using perturbation strength $\eta=10$. 

\section{Further plots of experiments}
\begin{figure*}[h]
    \centering
    \includegraphics[width=0.49\textwidth, height=4.3cm]{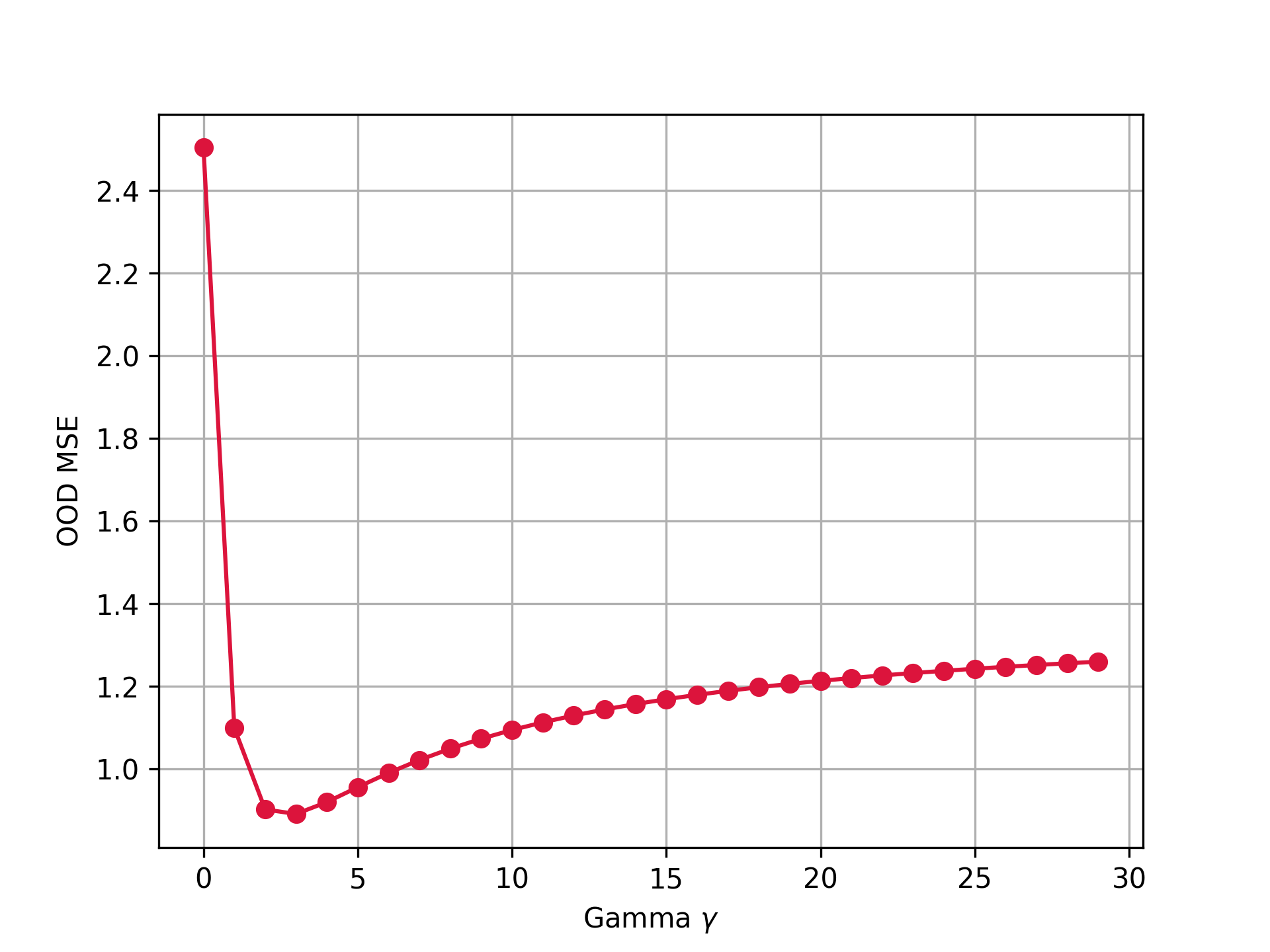}
    \includegraphics[width=0.49\textwidth, height=4.3cm]{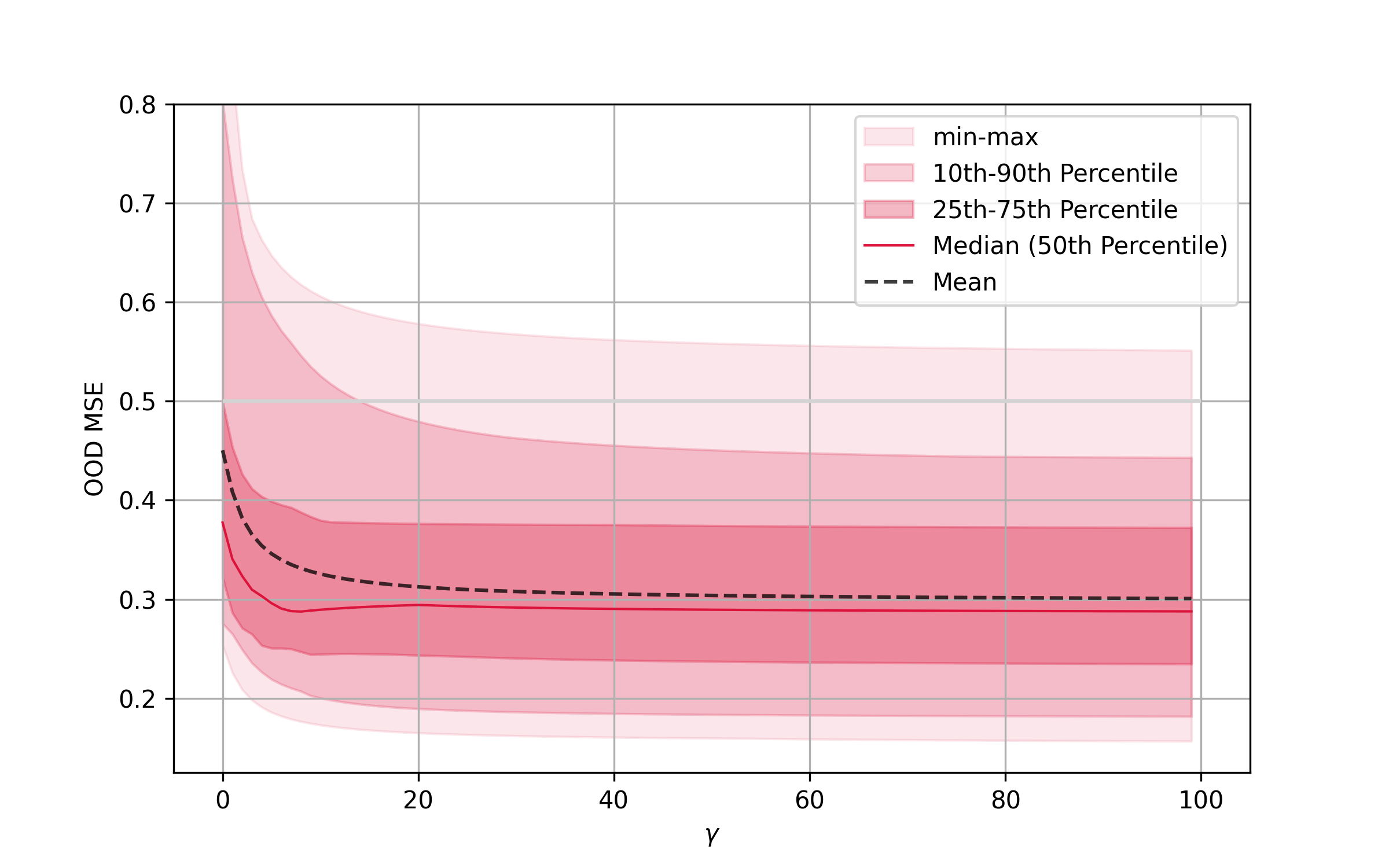}
\caption{Synthetically generated dataset of latent dimension two (left) \& single-cell dataset (right) - the 
panels illustrate the evolution of OOD MSE error of the proposed model in terms of chosen robustness radius $\gamma$. Notably, in the first case ,the \emph{finite} nature of the perturbation is clear, as the performance degrades for overly conservative values of $\gamma$. This occurrence is less clear in the second case, but still visible in the median. For the right panel, from top to bottom in shades of red: maximum, 90th, 75th quantile, median, 25th, 10th quantile, and minimum of the MSE across test environments. The black dashed line represents the mean.}\label{fig:testenvsplotqq}
\end{figure*}
The GMM regularization is partly motivated by the results of \citet{kivva2022identifiability}, who rely on the framework to generating the latents from a GMM. It can also be interpreted as a way to ensure a common invariant structure among the environments. For an ablation study of $\alpha$, consider the figure below.  
\begin{figure}[ht]\label{fig:alpha_ablations}
\centering
    \includegraphics[width=0.49\columnwidth]{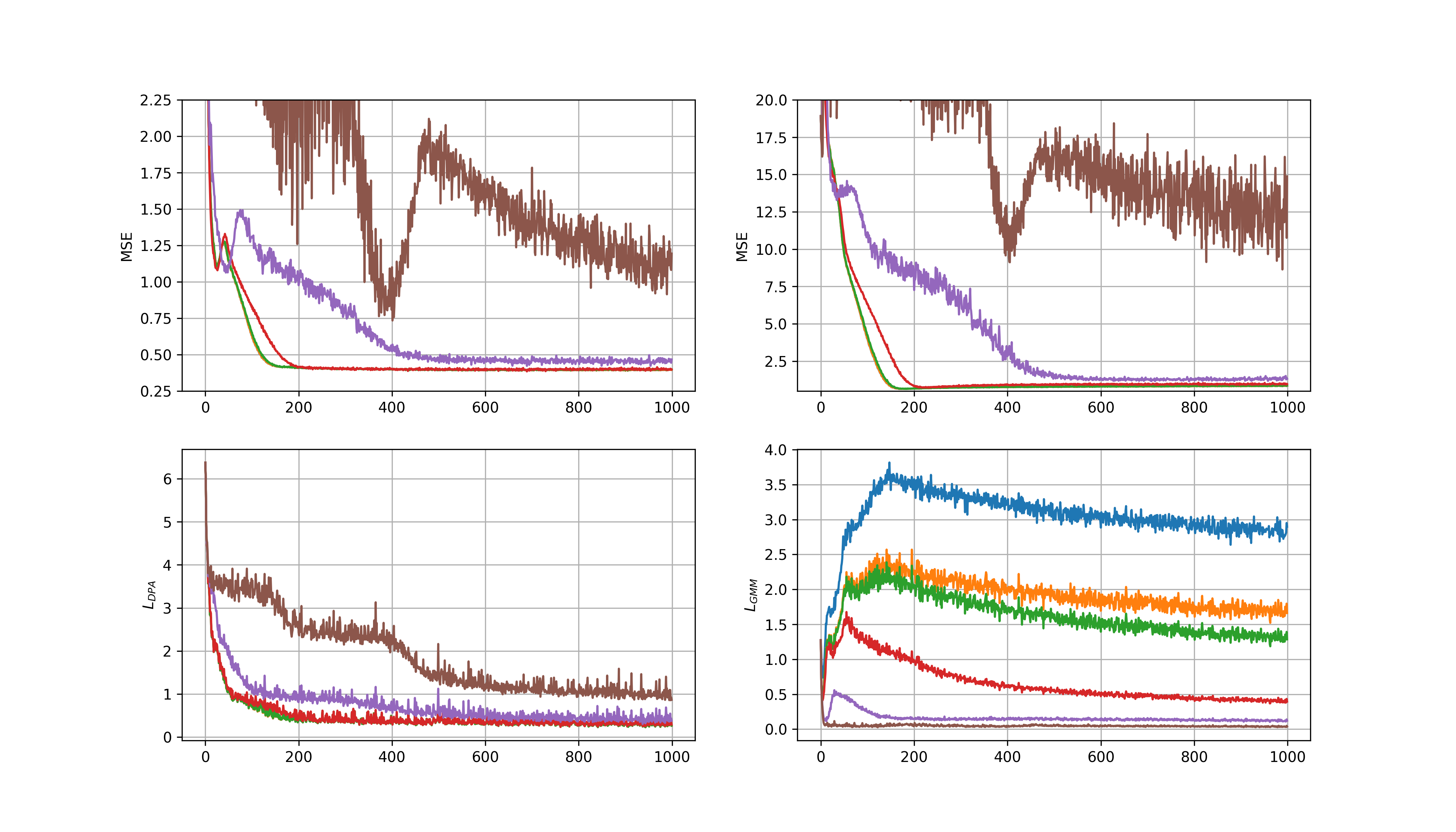}
    \includegraphics[width=0.49\columnwidth]{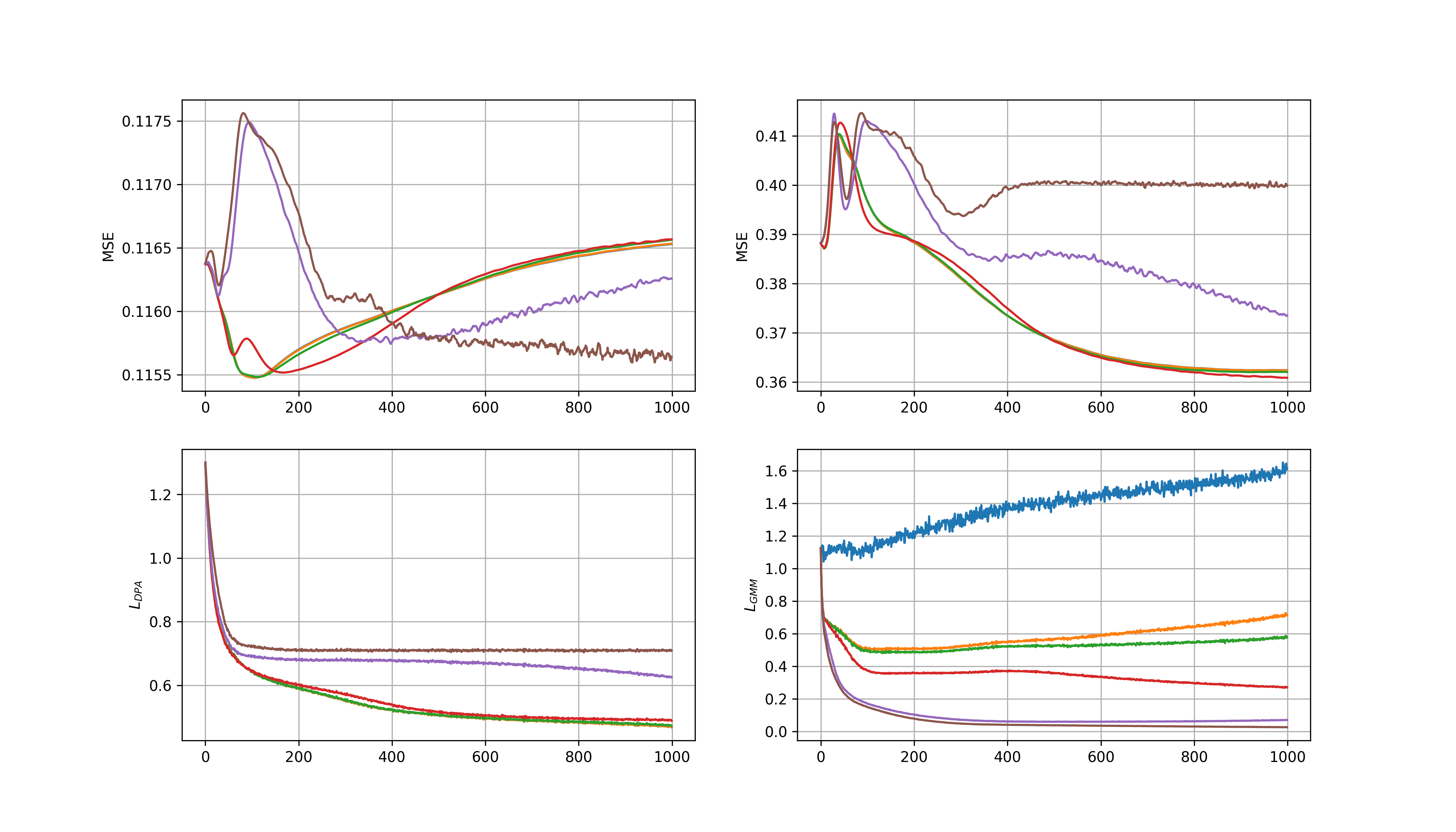}
\caption{Simulated (left), single-cell (right), learning curves over 1000 epochs for different values of $\alpha \in \{ 0, \frac{1}{1000}, \frac{1}{100}, \frac{1}{10}, 1, 10\}$ colored blue, orange, green, red, purple, brown, respectively. Considering only the lower pairs of plots, it is evident that $\alpha=\frac{1}{10}$ (red) achieves the best trade-off among the selected values for the optimized loss $L$. The upper row depicts performance in terms of training and test MSE, respectively.}
\end{figure}

%% file: Appendix_newproofs.tex
\begin{proof}[Proof of Propositon \ref{driglemma}] 
    Similar to the proof of Theorem 2 in \citet{shen2023causalityoriented}, replacing $X$ with $\phi^*(X)$. 
    \\
    From the SCM, it is known that $$\phi^*(X^v) = (\I -\Bb)^{-1}_{1:k,\cdot} (\varepsilon + v) \enspace \text{as well as} \enspace Y^v = (\I -\Bb)^{-1}_{k+1,\cdot} (\varepsilon + v)$$
    and by centering it is clear that, $\phi_c(X^v) = N(\I -\Bb)^{-1}_{1:k,\cdot} (\varepsilon + v)=N\phi^*(X^v) $ while $Y^v_c=Y^v$. 
    \\
    Let us denote $Y^v - b^\top \phi_c(X^v) = \bigg((\I -\Bb)^{-1}_{1:k,\cdot} - b^\top N(\I -\Bb)^{-1}_{k+1,\cdot}\bigg) (\varepsilon + v) = w^\top (\varepsilon + v)$. Then, rewriting the expectation under $\Pp_v$ gives 
\begin{align}
    \Ex_v [Y - b^\top \phi_c(X)]^2 
    &= \Ex [w^\top (\varepsilon + v)]^2 \\
    &= w^\top \Ex [(\varepsilon + v)(\varepsilon + v)^\top] w \\
    &= \Ex_0 [Y - b^\top \phi_c(X)]^2 + w^\top \Ex[vv^\top] w
\end{align}
Taking the supremum over $v\in C^{\gamma}$ on both sides only influences the second expression of the right hand side 
\begin{align}
    \sup_{v \in C^{\gamma}} w^\top \Ex_v [vv^\top] w 
    &= w^\top \bigg( S^0 + \gamma \sum_{e\in \E} \omega^e  (S^e - S^0 + \mu^e {\mu^e}^\top ) \bigg) w \\
    &= \Ex[w^\top\delta^0]^2 + \gamma \sum_{e\in\E} \omega^e (\Ex[w^\top\delta^e]^2-\Ex[w^\top\delta^0]^2)
\end{align}
Noticing  that by independence of $\varepsilon,\delta^e$ and the fact that $\varepsilon$ is centered $\Ex_0 [Y - b^\top \phi_c(X)]^2 + \Ex[w^\top\delta^0]^2 = \Ex [Y^0 - b^\top \phi_c(X^0)]^2$ holds, as well as
\begin{align}
    \Ex[w^\top\delta^e]^2-\Ex[w^\top\delta^0]^2 
    &= \Ex[w^\top\delta^e]^2 + \Ex[\varepsilon^\top\delta^e]^2 - \Ex[\varepsilon^\top\delta^e]^2 -\Ex[w^\top\delta^0]^2  \\
    &= \Ex[Y^e - b^\top \phi_c(X^e)]^2 - \Ex[Y^0 - b^\top \phi_c(X^0)]^2 
\end{align}
for all environments $e\in \E$, finishes the proof.
\end{proof}

The following technical lemma studies a helpful expression combining affine transforms of a random variable and a conditional expectation in case of normal distribution.
\begin{definition}\label{spherical_elliptical}\emph{(spherically symmetric distribution, elliptical symmetric distribution)}
    \begin{itemize}
        \item  A random vector $Y\in \R^d$ follows a spherically symmetric distribution, if there exists a scalar function $\psi$, such that $\psi (u^\top u) = \Ex[e^{iu^\top Y}]$. In this case, we denote $Y\sim S_d(\psi)$.  
        \item  A random vector $X\in \R^d$ follows an elliptically symmetric distribution with parameters $\mu, \Sigma, \psi$, if $X = \mu + AY $ in distribution, for $\Sigma = AA^\top, \enspace A\in \R^{d\times k}, \enspace \text{rk}(A)=k$ and $Y$ following a spherical distribution with a scalar function $\psi$. In this case, we denote $X\sim E_d(\mu, \Sigma, \psi)$.
    \end{itemize}
\end{definition}
Further examples of elliptical distributions include symmetric multivariate Laplace, Kotz, and logistic distribution.For a better overview of their properties, we refer to \citet{fang1990symmetric}. 
\begin{lemma}\label{gausslemma} 
Let $X\in \R^d$, and $X\sim E_d(\mu, \Sigma, \psi)$. Then, for a full-rank matrix $M\in \R^{k \times d}$ with $k\leq d$, the conditional expectation $\Ex [X | MX]$ is affine in $MX$, i.e. there is a matrix $A\in \R^{d\times k}$ and a vector $c\in \R^d$ such that $$\Ex [X | MX] = AMX + c.$$

\end{lemma}

\begin{proof}[Proof of Lemma \ref{gausslemma}]
Since $X\sim E_d(\mu, \Sigma, \psi)$, it is clear that $MX \sim E_k(M\mu, M\Sigma M^\top, \psi)$, and also  
\begin{align}
    \begin{pmatrix}
        X \\
        M X
    \end{pmatrix} 
    \sim E_{d+k}\Bigg(
    \begin{pmatrix}
        \mu \\
        M \mu
    \end{pmatrix}, 
    \begin{pmatrix}
        \Sigma  & \Sigma M ^\top \\
        M \Sigma & M \Sigma M^\top
    \end{pmatrix}, \psi
    \Bigg).
\end{align}
Furthermore, similar to the Gaussian case, the conditional $X|MX$ also follows an elliptically symmetric distribution and its expectation is known \citep{fang1990symmetric}, 
\begin{align}
    \Ex [X | MX] &= \mu +\Sigma M^\top (M \Sigma M^\top)^{-1} M(X-\mu) \\
    & = \Sigma M^\top (M \Sigma M^\top)^{-1} MX + (\I - \Sigma M^\top (M \Sigma M^\top)^{-1} M)\mu,
\end{align}
which is an affine function of $MX$.
\end{proof}

\begin{corollary}
    In case $X$ is a centered random variable as above, then $\Ex[X| MX]$ is a linear function of $MX$.
\end{corollary}

\begin{remark}
    The matrix $\Sigma M^\top (M \Sigma M^\top)^{-1} M$ is an (oblique) projection onto the row space of $M$, transformed by $\Sigma$. 
\end{remark} 
\begin{proof}[Proof of Theorem \ref{nonlinrob}]
\begin{align}
    \sup_{v\in C^\gamma} \Ex_v [Y - \widehat{b}^\top  \widehat{\phi}_c(X)]^2
    &=  \min_{\substack{b \in \R^k}} \enspace \sup_{v\in C^\gamma} \Ex_v [Y - b^\top N \phi^*(X)]^2  \\
    & = \min_{\substack{b \in \R^k}} \enspace \sup_{v\in C^\gamma} \Ex_v [Y - b^\top \phi^*(X) ]^2  
\end{align}
Now consider the minimizer of $\Ex_v [Y-f(X)]^2$, namely $\Ex_v [Y|X]$. Since $X,Y$ are $d$-separated by $\phi^*(X)$ in the SCM, we see how $\Ex_v [Y|X] = \Ex_v [Y|X, \phi^*(X)] = \Ex_v [Y|\phi^*(X)]$ and moving forward, it becomes clear that 
\begin{align}
    \Ex_v [Y|\phi^*(X)] & = (\I-\Bb)^{-1}_{k+1,\cdot} \Ex \bigg[\varepsilon + v \bigg| (\I-\Bb)^{-1}_{1:k,\cdot} (\varepsilon + v)\bigg] \\
  \text{by Lemma \ref{gausslemma}}
    & = (\I-\Bb)^{-1}_{k+1,\cdot} A M(\varepsilon + v) + (\I-\Bb)^{-1}_{k+1,\cdot}(\I - A M)\mu_v, 
\end{align} 
for $A=\Sigma M^\top (M \Sigma M^\top)^{-1}, M=(\I - \Bb)^{-1}_{1:k,\cdot}, \Sigma = \Sigma_\varepsilon + \Sigma_v 
$.   
$AM$ is a projection matrix projecting onto the row-space of $M$, transformed by $\Sigma$. In case $\mu_v$ can be represented as $\Sigma M^\top \alpha$ for some $\alpha \in \R^k$, then it is clear that $(\I - AM)\mu_v=0$ is
always the case. 
Hence, $\Ex_v[Y|X]$ is a linear function of $\phi^*(X)$ as as whole, it is true that for some $a\in \R^{k}$ 
\begin{align}
    \min_{f\in \Le}\Ex_v [Y-f(X)]^2  & = \Ex_v \bigg[Y-\Ex_v [Y|X]\bigg]^2 \\
    & = \Ex_v [Y - a^\top \phi^*(X)]^2 \\
    & \geq  \min_{b \in \R^k} \Ex_v [Y - b^\top \phi^*(X) ]^2
\end{align}
The reverse inequality between the two is always true. 
Since for any $v\in C^{\gamma}$ $$\min_{f\in \Le}\Ex_v [Y-f(X)]^2 = \min_{b\in \R^k} \Ex_v [Y - b^\top \phi^*(X)]^2,$$ it also holds $$\min_{f\in \Le}\supv_v [Y-f(X)]^2 = \min_{b\in \R^k} \supv_v [Y - b^\top \phi^*(X)]^2.$$
\end{proof}